\documentclass[sigconf]{acmart}
\usepackage{tikz}
\usepackage{amsmath,amsfonts,bm}
\usepackage{mathtools}
\usepackage{comment}
\usepackage[font=bf]{caption}
\usepackage{graphicx}
\usepackage{booktabs}
\usepackage{array,graphicx,multirow}
\usepackage{dblfloatfix}
\usepackage{soul}
\usepackage{hyperref}
\usepackage{hhline}
\usepackage{algorithm}
\usepackage{algorithmic}
\usepackage{enumitem}
\usepackage{diagbox}

\usepackage{cancel}
\theoremstyle{plain}

\theoremstyle{definition}

\usepackage{subfig}
\usepackage[normalem]{ulem}
\def\removesen{1}
\if 1\removesen
\newcommand\autost[1]{\textcolor{red!40!white}{\sout{#1}}}
\else
\newcommand\autost[1]{}
\fi

\usepackage{pifont}

\usepackage{lipsum}
\usepackage{dirtytalk}
\definecolor{cadmiumgreen}{rgb}{0.0, 0.42, 0.24}

\usepackage[framemethod=default]{mdframed}
\newmdenv[linecolor=black,backgroundcolor=gray!20]{myframe}
\usepackage[skins]{tcolorbox}
\newtcolorbox{mytcolorbox}[2][width=0.975\columnwidth,halign=flush center,arc is angular]{colback=gray!10,colframe=black,fonttitle=\bfseries,coltitle=black,colbacktitle=white,enhanced,attach boxed title to top center={yshift=-2mm},title={#2},#1}

\usepackage{fancyhdr,blindtext}
\setcopyright{none}
\acmYear{2022}

\usepackage{flushend}

\usepackage{fancyhdr}

\begin{document}

\title{DETERRENT: \underline{Dete}cting T\underline{r}ojans using \underline{Re}inforceme\underline{nt} Learning}

\author{Vasudev~Gohil, Satwik~Patnaik, Hao~Guo, Dileep~Kalathil, and Jeyavijayan~(JV)~Rajendran}
\affiliation{\institution{Electrical \& Computer Engineering, Texas A\&M University, College Station, Texas, USA}}
\email{{gohil.vasudev, satwik.patnaik, guohao2019, dileep.kalathil, jv.rajendran} @tamu.edu}

\renewcommand{\shortauthors}{Gohil et al.}

\begin{abstract}
Insertion of hardware Trojans (HTs) in integrated circuits is a pernicious threat. 
Since HTs are activated under rare trigger conditions, detecting them using random logic simulations is infeasible.
In this work, we design a reinforcement learning (RL) agent that circumvents the exponential search space and returns a minimal set of patterns that is most likely to detect HTs. 
Experimental results on a variety of benchmarks demonstrate the efficacy and scalability of our RL agent, which obtains a significant reduction ($169\times$) in the number of test patterns required while maintaining or improving coverage ($95.75\%$) compared to the state-of-the-art techniques.
\end{abstract}

\keywords{Reinforcement Learning, Hardware Trojans}

\maketitle

\renewcommand{\headrulewidth}{0.0pt}
\thispagestyle{fancy}
\lhead{}
\rhead{}
\chead{\copyright~2022 IEEE.
This is the author's version of the work.
The definitive Version of Record is published in
2022 Design Automation Conference (DAC), DOI 10.1145/3489517.3530518}
\cfoot{}

\section{Introduction}
\label{sec:Introduction}

Reinforcement learning (RL) helps a computing system (a.k.a. agent) to learn by its own experience through exploring and exploiting the underlying environment.
Over time, the agent takes optimal actions in sequence, even with limited or no knowledge regarding the environment.
From a cybersecurity perspective, such RL agents are attractive as they can generate optimal defense techniques in an unknown adversarial environment.
Given the latest improvements in RL algorithms, these agents can efficiently navigate high-dimensional search space to find optimal actions.
Hence, researchers have used RL agents to develop promising approaches for several security problems, including intrusion detection~\cite{RL_intrusion_detection}, fuzzing~\cite{RL_Fuzzing,RL_Fuzzing_USENIX}, and developing secure cyber-physical systems~\cite{RL_CPS1,RL_CPS2,nguyen2019deep_new}.
However, research in hardware security is still in its infancy to reap the power of RL in developing optimal defenses in adversarial environments.
In this work, we showcase how RL can be used to efficiently detect hardware Trojans (HTs).
Out of the many problems in hardware security, the HT detection problem presents significant computational challenges to the defender in detecting them in an unknown environment (i.e., HT-infected design).

The increasing cost of integrated circuit (IC) manufacturing has forced semiconductor companies to send their designs to untrusted, off-shore foundries.
Malicious components known as HTs inserted during the fabrication stage can leak secret information, degrade performance, or cause a denial of service.

\subsection{Hardware Trojans}
\label{sec:HTs}

An HT consists of two components: \textit{trigger} and \textit{payload}. 
When the trigger is activated, the payload causes a malicious effect in the design. 
Figure~\ref{fig:Trojan} illustrates an HT that flips an output upon trigger activation. 
The trigger comprises multiple nets, called \textit{select nets}, in the design.
For instance, the adversary can choose the select nets so that the trigger gets activated only under extremely rare conditions.
This is achieved by determining a \textit{rareness threshold}\footnote{Rareness threshold is the probability below which nets are classified as rare nets.} and constructing the trigger using the corresponding \textit{rare nets}.

\begin{figure}[tb]
\includegraphics[width=0.475\textwidth,trim={0 0.8cm 0 0.6cm},clip]{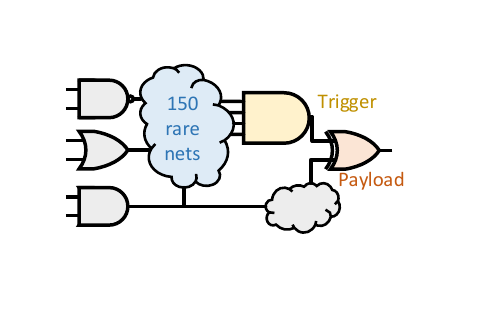}
\caption{Example of an HT in a design with $150$ rare nets.}
\label{fig:Trojan}
\end{figure}

Detecting HTs is difficult since they are designed to be stealthy~\cite{xiao2016Trojan_survey}.
Consider the example in Figure~\ref{fig:Trojan} with $150$ rare nets.
Four of them are used for the trigger.
Thus, the defender needs to check up to ${}^{150}C_{4} \approx 20\times10^6$ different combinations of rare nets, which is extremely challenging. Such a large space makes it difficult even for conventional automatic test pattern generation (ATPG) tools~\cite{TestMAX} to activate the trigger.

\subsection{Hardware Trojan Detection Techniques}
\label{sec:detecting_HTs}

One can classify the HT detection techniques under two broad categories: logic testing and side-channel analysis.
Logic testing involves the application of test patterns to the HT-infected design to activate the trigger~\cite{chakraborty2009mero,TARMAC_TCAD,pan2021automated}. 
However, activating an extremely rare trigger is challenging because the possible combinations of rare nets are extensive. 
On the other hand, side-channel-based detection techniques detect HTs based on the differences in the side-channel measurements (such as power or timing) between the golden (i.e., HT-free) design and an HT-infected design~\cite{narasimhan2012hardware,huang2016mers,huang2018scalable,lyu2019efficient}. 
However, since HTs have an extremely small footprint compared to the overall size of the design, their impact on side-channel metrics is usually negligible and concealed under process variation and environmental effects~\cite{rai2009performance}. 
We refer interested readers to~\cite{xiao2016Trojan_survey} for a detailed survey on HTs and HT detection techniques.

Note that activating the trigger is not only essential for logic testing techniques but also helpful for side-channel-based techniques because activating the trigger leads to an increase in the side-channel footprint of the HT, making it easier to detect~\cite{TARMAC_TCAD}.

Although activating the trigger is critical, it is difficult to do so efficiently. 
Consider Figure~\ref{fig:Trojan}; the defender needs up to $20\times 10^6$ test patterns to guarantee trigger activation because the defender does not know which rare nets make the trigger.
Next, we outline the ideal characteristics required from any technique for activating the trigger. \textbf{(1) High trigger activation rate}: The technique should activate a large number of trigger conditions to detect HTs successfully.\footnote{Trigger activation rate, i.e., the proportion of trigger conditions activated by a set of test patterns, is also called trigger coverage.} 
\textbf{(2) Small test generation time}: The time required to generate the test patterns should not be large; otherwise, the technique will not be scalable to larger designs. 
\textbf{(3) Compact set of test patterns}: The number of test patterns required to activate the trigger conditions should be small. 
A large number of test patterns affect the testing cost adversely. 
\textbf{(4) Feedback-guided approach}: The technique should analyze the test patterns and their impact on the circuit to generate new test patterns, thereby reducing the test generation time and the size of the test set.

\subsection{Prior Works and Their Limitations}
\label{sec:prior_work}

\noindent\textbf{MERO} generates test patterns that activate each rare net $N$ times~\cite{chakraborty2009mero}. 
The hypothesis is that if all the rare nets are activated $N$ times, the test patterns are likely to activate the trigger. 
The algorithm starts with a large pool of random test patterns and iteratively performs circuit simulation to keep track of the number of rare nets that get activated.
While MERO provides moderate performance for small benchmarks, it fails for large benchmarks. For instance, the trigger coverage of MERO for the \texttt{MIPS} processor is only 0.2\%~\cite{TARMAC_TCAD}, as it violates the characteristics (1), (2), (3), and (4) mentioned above.

\noindent\textbf{TARMAC} overcomes the limitations of MERO by transforming the problem of test pattern generation into a clique cover problem~\cite{TARMAC_TCAD}. It iteratively finds maximal cliques of rare nets that satisfy their rare values. 
By not relying on brute force, TARMAC outperforms MERO by a factor of 71$\times$ on average.
However, the performance of TARMAC is sensitive to randomness since the algorithm relies on randomly sampled cliques.
Although the test generation time for TARMAC is short, it violates characteristics (3) and (4).

\noindent\textbf{TGRL} uses RL along with a combination of rareness and testability measures to overcome the limitations of TARMAC~\cite{pan2021automated}. 
TGRL achieves better coverage than TARMAC and MERO while reducing the run-time. 
However, it still violates characteristic (3), as evidenced by our results in Section~\ref{sec:results}.

\subsection{Our Contributions}
\label{sec:research_contributions}

As discussed above, all existing techniques for trigger activation fall short on one or more fronts. 
In this work, we propose a new technique that is designed to satisfy all four ideal characteristics.
We model the test generation problem for HT detection as an RL problem because test generation involves searching a large space to find an optimal set of test patterns.
This is exactly what RL algorithms do: they navigate large search spaces to find optimal solutions.
However, there are several challenges that need to be overcome to realize a practical and scalable RL agent, such as (i)~large amount of training time required for large designs, (ii)~the agent needs to be efficient while choosing actions, and (iii)~some challenging benchmarks require smart fine-tuning.
We provide further details on how we overcome these challenges in Section~\ref{sec:proposed_framework}. 
The primary contributions of our work are as follows.

\begin{itemize}[leftmargin=*]

\item We develop an RL technique that is efficient in activating rare trigger conditions, thereby addressing the limitations of the state-of-the-art HT detection techniques.

\item We overcome several challenges to make our technique scalable to a large design like the \texttt{MIPS} processor.

\item We perform an extensive evaluation on diverse benchmarks and demonstrate the capability of our technique, which outperforms the state-of-the-art logic-testing techniques on all benchmarks. 

\item Our technique provides two orders of magnitude ($169\times$) reduction in the size of the test set compared to existing techniques.

\item Our technique maintains similar trigger coverage ($\leq 2\%$ drop) with increasing number of rare nets, whereas the state-of-the-art technique's performance drops to $0\%$.

\item Our technique maintains similar trigger coverage ($\leq 2\%$ drop) for at least $64\times$ more potential trigger conditions.
    
\item We release our benchmarks and test patterns~\cite{DETERRENT-git}.

\end{itemize}
\section{Assumptions and Background}
\label{sec:background}

\subsection{Threat Model}
\label{sec:threat_model}

We assume the standard threat model used in logic testing-based HT detection~\cite{chakraborty2009mero,TARMAC_TCAD,pan2021automated}. 
We assume that the adversary inserts HTs in rare nets of the design to remain stealthy. The defender's (i.e., our) objective is to generate a minimal set of test patterns that activate unknown trigger conditions. We generate test patterns using only the golden (i.e., HT-free) netlist. 

\subsection{Reinforcement Learning}
\label{sec:RL_background}

RL is a machine learning methodology where an intelligent agent learns to navigate an environment to maximize a cumulative reward.
It is formalized as a Markov decision process. 
An RL agent interacts with the environment in discrete time steps. 
At each step, the agent receives the current state and the reward, and it chooses the action which is sent to the environment. 
The environment moves the agent to a new state and provides a reward corresponding to the state transition and action. 
The aim of the RL agent is to learn a policy $\pi$ that maximizes the expected cumulative reward. 
The policy maps state-action pairs to probabilities of taking that action in a given state. 
The agent learns the optimal or near-optimal policy in a trial-and-error method by interacting with the environment.
\section{DETERRENT: Detecting Trojans using Reinforcement Learning}
\label{sec:proposed_framework}

We now formulate the trigger activation problem as an RL problem, but it suffers from challenges related to scalability, efficiency, and poor performance. We then address these challenges and devise a final RL agent that outperforms all existing techniques.
 
\subsection{A Simple Formulation}
\label{sec:initial_formulation}

As shown in Figure~\ref{fig:Trojan}, to activate the trigger, the defender has to apply an input pattern that forces all four rare nets to take their rare values simultaneously,\footnote{For the sake of conciseness, henceforth, we shall use the phrase ``activate the rare nets'' instead of ``force the rare nets to take their rare values.''} but the defender does not know which four rare nets constitute the trigger.
A na\"ive solution is to generate one input pattern for each combination of four rare nets. 
Such an approach would require up to ${^r}C_{4}$ test patterns ($r$ is the total number of rare nets), which would be infeasible to employ in practice. 
However, one input pattern can activate multiple different combinations of rare nets simultaneously. 
So, we need to find a minimal set of input patterns that can collectively activate all combinations of rare nets. This problem is a variant of the set-cover problem, which is NP-complete~\cite{cormen2009introduction}. 
We call a set of rare nets \textit{compatible} if there exists an input pattern that can activate all the rare nets in the set simultaneously. 
Thus, our objective is to develop an RL agent that generates maximal sets of compatible rare nets.

We now map the trigger activation problem into an RL problem by formulating it as a Markov decision process.

\begin{itemize}[leftmargin=*]

\item \textbf{States $\mathcal{S}$} is the set of all subsets of the rare nets. An individual state $s_t$ represents the set of compatible rare nets at time $t$.

\item \textbf{Actions $\mathcal{A}$} is the set of all rare nets. An individual action $a_t$ is the rare net chosen by the agent at time $t$.

\item \textbf{State transition $P(s_{t+1}|a_t,s_t)$} is the probability that action $a_t$ in state $s_t$ leads to the state $s_{t+1}$. 
In our case, if the chosen rare net (i.e., the action) is compatible with the current set of rare nets (i.e., the current state), we add the chosen rare net to the set of compatible rare nets (i.e., the next state). 
Otherwise, next state remains the same as the current state. 
Thus, in our case, the state transition is deterministic, as shown below.
    \[
    s_{t+1}=
    \begin{dcases}
        \{a_t\} \cup s_t, & \text{if } a_t\text{ is compatible with }s_t\\
        s_t,              & \text{otherwise}
    \end{dcases}
    \]

\item \textbf{Reward function $R(s_t,a_t)$} is equal to the square of the size of the next state for compatible states, and 0 otherwise.
    \[
    R(s_{t},a_t)=
    \begin{dcases}
        |\{a_t\} \cup s_t|^2 = |s_{t+1}|^2, & \text{if } a_t\text{ is compatible with }s_t\\
        0,              & \text{otherwise}
    \end{dcases}
    \]
The reward is designed so that the agent tries to maximize the size of the state, i.e., the number of compatible rare nets. 
We square the reward at each step, but any power greater than 1 would be appropriate since we want the reward function to be convex to account for the fact that as the size of the state grows, the difficulty of finding a new compatible rare net increases.
    
\item \textbf{Discount factor $\gamma$} $(0 \leq \gamma \leq 1)$ indicates the importance of future rewards relative to the current reward.
\end{itemize}

The initial state $s_0$ is a singleton set containing a randomly chosen rare net.
At each step $t$, the agent in state $s_t$ chooses an action $a_t$, arrives in the next state $s_{t+1}$ according to the state transition rules, and receives a reward $r_{t}$. 
This cycle of state, action, reward, and next state is repeated $T$ times, and this constitutes one \textit{episode}.
At the end of each episode, the state of the agent reflects the rare nets that are compatible.\footnote{For software implementation, we represent the states (which are defined as sets) as binary vectors, with each element on the vector indicating whether the corresponding rare net is present in the state or not.} Since the state and action spaces are discrete, we train our agent using the Proximal Policy Optimization (PPO) algorithm with default parameters unless specified otherwise~\cite{schulman2017proximal}.

Once the agent returns the maximal sets of compatible rare nets after training, we pick the $k$ largest distinct sets and generate the test patterns corresponding to those sets using a Boolean satisfiability (SAT) solver. $k$ is a hyperparameter of our technique. 

Our experiments indicate that this simple agent performs well on small benchmarks. 
But, for larger benchmarks like the \texttt{MIPS} processor from OpenCores~\cite{OpenCores_MIPS} we obtain low trigger coverage ($\approx$70\% after training for 12 hours). 
We analyzed the basic architecture in detail, and it faces certain challenges which are presented next.

\begin{figure}[htb]
\centering
\includegraphics[width=0.475\textwidth,trim={0.2cm 0.4cm 0 0},clip]{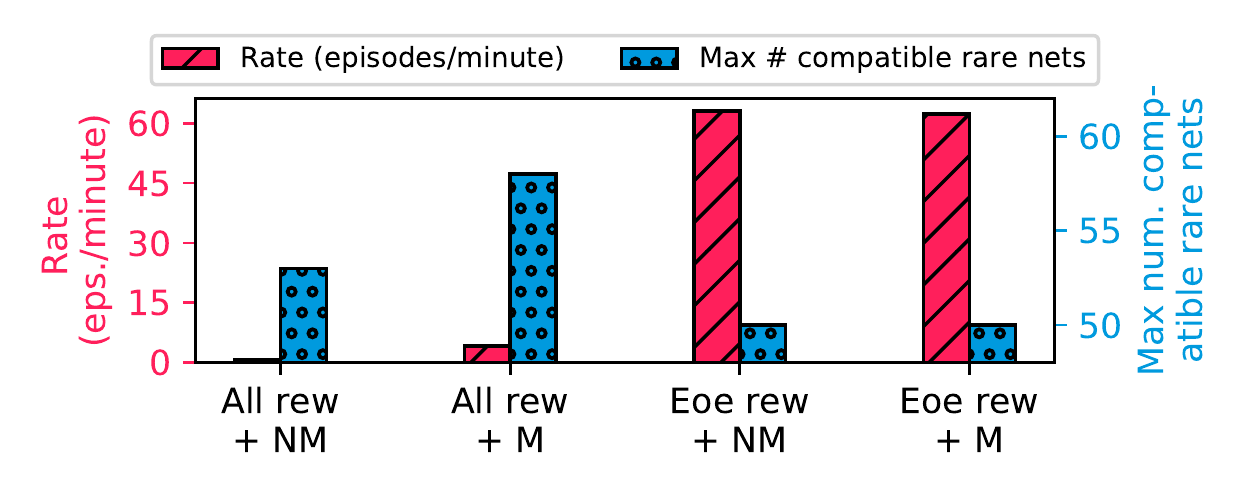}
\caption{Combinations of reward and masking methods for \texttt{MIPS}. Eoe: End-of-episode, M: Masking, NM: No masking
}
\label{fig:challenge_sol2}
\end{figure}

\subsection{End-of-Episode Reward Computation}

\noindent\textbf{Challenge 1: Large training time.} The basic architecture requires computing the reward for each time step, which involves checking if the selected action is compatible with the current state or not. 
For a large benchmark like the \texttt{MIPS} processor, the check takes a few seconds (because of the large number of gates in the benchmark) each time, and the agent requires millions of steps to learn. Hence, the training time becomes prohibitively large.

\noindent\textbf{Solution 1.} To address challenge 1, we reduce the frequency of reward computation by computing it only at the end of the episode. 
At all intermediate steps, the reward is set to 0.
While this approach speeds up the training by a factor of $\approx86\times$, the rewards become sparse, and it affects the performance of our agent. However, the impact on performance is only $5.6\%$, as shown in Table~\ref{tab:challenge_sol1}. 

\begin{table}[htb]
\caption{Comparison of training rates for the reward methods for the \texttt{MIPS} benchmark: all steps vs. end-of-episode.}
\label{tab:challenge_sol1}
\resizebox{0.45\textwidth}{!}{%
\begin{tabular}{cccc} \toprule
\multirow{2}{*}{Method} & \multirow{2}{*}{\begin{tabular}[c]{@{}c@{}}Max. \# compatible \\ rare nets\end{tabular}} & \multicolumn{2}{c}{Rate}  \\ \cmidrule(lr){3-4}
& & (steps/min) & (eps./min) \\ \midrule
Reward at all steps & 53 & 108 & 0.72\\ 
End-of-episode reward & 50 & 9387 & 63\\ 
Improvement & -5.6\% & 86.91$\times$ & 87.5$\times$\\ \bottomrule
\end{tabular}
}
\end{table}

\subsection{Masking Actions for Efficiency} \noindent\textbf{Challenge 2: Wasted efforts in choosing actions.} Another challenge that the basic architecture suffers from is inefficiency in choosing actions. At each step, the actions available to the agent remain the same, irrespective of the state of the agent. This leads to situations where the agent chooses an action that has already been chosen in the past, or
that is known to be not compatible with at least one of the rare nets in the current state. 
Hence, the time spent by the agent on such steps is wasted. 

\begin{figure}[tb]
\centering
\includegraphics[width=0.475\textwidth,trim={0 0.35cm 0 0.35cm},clip]{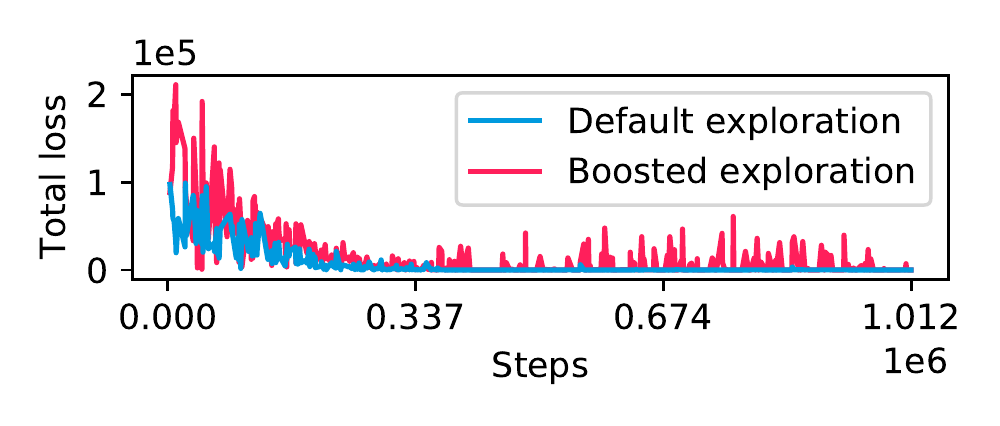}
\caption{Total loss trends in  \texttt{c2670}:  default vs. boosted exploration.}
\label{fig:challenge_sol3}
\end{figure}

\begin{figure*}[tb]
\centering
\includegraphics[width=\textwidth,trim={0.75cm 0.73cm 0.6cm 0},clip]{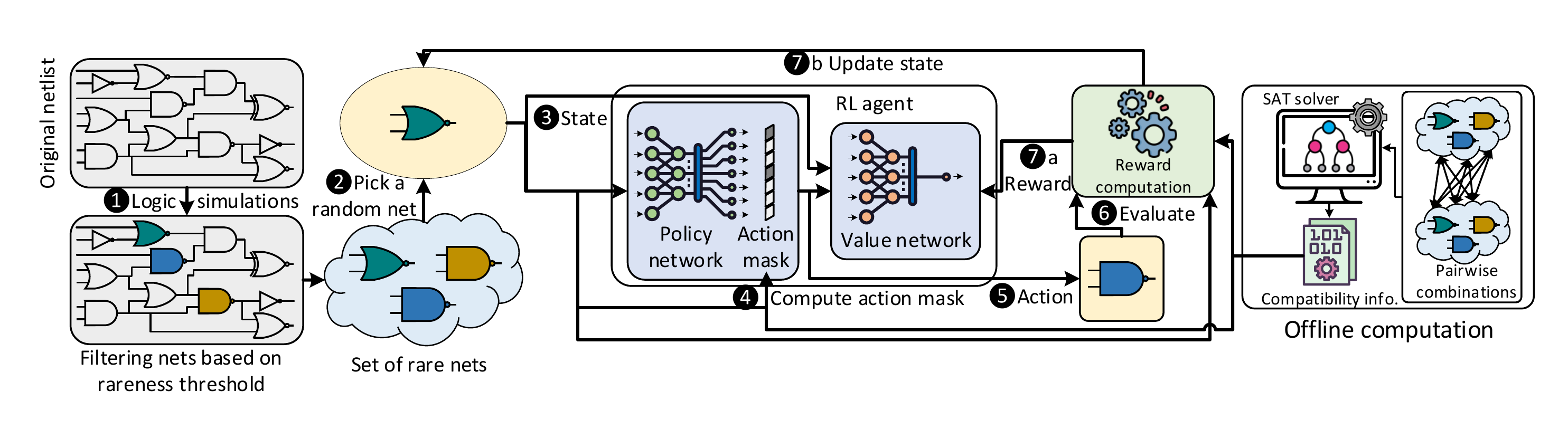}
\caption{Architecture of DETERRENT.}
\label{fig:RL_pipeline}
\end{figure*}

\noindent\textbf{Solution 2.} To increase the efficiency of the agent in choosing actions, we mask the actions available to the agent based on the state at any given time step. This ensures that at each time step, the agent only chooses actions that lead it to a new state.
Additionally, reward computation also becomes less sparse because episode lengths reduce due to masking (episode ends when there are no available actions). Since we are eliminating actions from the state space, one may wonder if this approach may eliminate optimal actions. We now prove that this is not possible for our problem formulation. 

\begin{theorem}
Masking actions does not prevent our agent from learning anything that it could have learned otherwise.
\end{theorem}

\begin{proof}
Let $\mathcal{P}'$ and $\mathcal{P}$ denote an agent that masks and does not mask actions, respectively. Suppose both $\mathcal{P}$ and $\mathcal{P}'$ are in state $s$. Let $\mathcal{A}$ denote the complete set of actions, and $\mathcal{A}_s$ denote the set of masked actions for state $s$. So, $\mathcal{A}_s = \{i|i \text{ is compatible with }s \text{ and } i \notin s\}$ and $\mathcal{A}_s \subseteq \mathcal{A}$.
If $\mathcal{P}$ chooses an action $a \in \mathcal{A} \setminus \mathcal{A}_s$ (i.e., an action in the set difference), then $\mathcal{P}$ will stay in the same state because the rare net corresponding to such an action $a$ would either be incompatible with $s$ or it would already be in $s$. On the other hand, for any action $a' \in \mathcal{A}_s$ chosen by $\mathcal{P}$, agent $\mathcal{P}'$ can also choose the same action $a'$ since it is in $\mathcal{A}_s$. Hence, masking does not prevent our agent from learning anything that the corresponding unmasked agent could have learned.
\end{proof}

To enable masking, we compute pairwise compatibility of all rare nets using a SAT solver before training.
Since the compatibility computation for each unique pair is independent, we parallelize it across 64 processes to reduce the runtime. 
During training, for a given state $s$ (i.e., set of compatible rare nets at the current step), all actions (i.e., rare nets) that are not compatible with any of the rare nets in $s$ are masked off, and hence, are not chosen.

To design the best architecture, we implemented agents with all combinations of reward methods (at all steps and end-of-episode) and masking (with and without). The results in Figure~\ref{fig:challenge_sol2} demonstrate that to obtain the maximum number of compatible rare nets, the optimal architecture should mask actions based on state and provide rewards at each time step.

\subsection{Boosting Exploration}

\noindent\textbf{Challenge 3: Convergence to local optima.} Since the agent's objective is to generate maximal sets of rare nets, for certain benchmarks (for instance, \texttt{c2670}), the agent gets stuck in local optima. 
In other words, the agent quickly learns to capitalize on sub-optimal sets of compatible rare nets, thereby missing out on the diversity of the sets of compatible rare nets, resulting in poor trigger coverage. 

\noindent\textbf{Solution 3.} To force the agent to explore, we (1) include an entropy term in the loss function of the agent and (2) control the smoothing parameter that affects the variance of the loss calculation. 

To implement (1), we modify the total loss function to $l = l_{\pi} + c_\epsilon \times l_{\epsilon} + c_v \times l_{v}$, where $l$ is the total loss, $l_{\pi}$ is the loss of the policy network, $l_\epsilon$ is the entropy loss, $l_v$ is the value loss, and $c_\epsilon$ and $c_v$ are the coefficients for the entropy and value losses, respectively. We set $c_{\epsilon} = 1$. The entropy loss is inversely proportional to the randomness in the choice of actions. To implement (2), we set the parameter $\lambda$ for policy loss $l_{\pi}$ in PPO to $0.99$. This leads to variance in the loss calculation and hence in the actions chosen by the agent.

Thus, we penalize the agent for having less variance in its choice of actions. Hence, the agent is forced to explore more and is likely to converge to a better state, i.e., a state with more compatible rare nets.
Figure~\ref{fig:challenge_sol3} shows that by modifying the loss function and the smoothing parameter in PPO, the loss does not become $0$ quickly, forcing the agent to explore more.

\subsection{Putting it All Together}
\label{sec:final}

The final architecture of DETERRENT is illustrated in Figure~\ref{fig:RL_pipeline}. In an offline phase, we find the rare nets of the design and generate pairwise compatibility information for them in a parallelized manner. 
Then, for each episode, the agent starts with a random rare net and takes an action according to the policy (a neural network) and the action mask.
The masked action is evaluated to produce a reward for the agent, and the agent moves to the next state. 
This procedure repeats for $T$ steps (i.e., an episode). 
Internally, after a certain number of episodes, the PPO algorithm translates the rewards into losses (depending on the output of the policy network, which generates actions, and the value network, which predicts the expected reward of the action), which are used to update the parameters of the policy and value networks. 
Eventually, when the agent has learned the task, the losses become negligible, and the reward saturates. 
Once the RL agent gives us the maximal sets of compatible rare nets, we pick the $k$ largest distinct sets and generate the test patterns, one for each of those sets, using a SAT solver.
\section{Experimental Evaluation}\label{sec:results}

\subsection{Experimental Setup}
\label{sec:setup}

\begin{table*}[ht]
\centering
\caption{Comparison of trigger coverage (Cov. (\%)) and test length of DETERRENT with random simulations, Synopsys TestMAX~\cite{TestMAX}, TARMAC~\cite{TARMAC_TCAD}, and TGRL~\cite{pan2021automated}. Evaluation is done on 100 random four-width triggered HT-infected netlists.}
\label{tab:my-table}
\resizebox{\textwidth}{!}{
\begin{tabular}{cccccccccccccc}
\toprule
\multirow{3}{*}{Design}& \multirow{1}{*}[-0.1cm]{Number} & \multirow{3}{*}{\# Gates} & \multicolumn{2}{c}{Random} & \multicolumn{2}{c}{TestMAX~\cite{TestMAX}} & \multicolumn{2}{c}{TARMAC~\cite{TARMAC_TCAD}} & \multicolumn{2}{c}{TGRL~\cite{pan2021automated}} & \multicolumn{3}{c}{\textbf{DETERRENT (this work)}}\\
\cmidrule(lr){4-5} 
\cmidrule(lr){6-7} 
\cmidrule(lr){8-9} \cmidrule(lr){10-11} \cmidrule(lr){12-14}    
 & of rare & & Test & Cov. & Test & Cov. & Test & Cov.& Test & Cov. & Test & Patterns Red./ & Cov. \\
& nets & & Length & (\%) & Length & (\%) & Length & (\%) & Length & (\%) & Length & TARMAC \& TGRL  & (\%) 
\\ \midrule
\texttt{c2670} & 43 & 775 & \multicolumn{1}{c}{5306} & 10 & \multicolumn{1}{c}{89} & 27 & \multicolumn{1}{c}{5306} & 100 & \multicolumn{1}{c}{5306} & 96 & \multicolumn{1}{c}{8} & \multicolumn{1}{c}{\textbf{663.25$\bm{\times}$}} & 100 \\
\texttt{c5315} & 165 & 2307 & \multicolumn{1}{c}{8066} & 37 & \multicolumn{1}{c}{103} & 5 & \multicolumn{1}{c}{8066} & 61 & \multicolumn{1}{c}{8066} & 94 & \multicolumn{1}{c}{1585} & \multicolumn{1}{c}{\textbf{5.08$\bm{\times}$}} & 99 \\
\texttt{c6288} & 186 & 2416 & \multicolumn{1}{c}{3205} & 54 & \multicolumn{1}{c}{38} & 4 & \multicolumn{1}{c}{3205} & 100 & \multicolumn{1}{c}{3205} & 85 & \multicolumn{1}{c}{2096} & \multicolumn{1}{c}{\textbf{1.52$\bm{\times}$}} & 99 \\ 
\texttt{c7552} & 282 & 3513 & \multicolumn{1}{c}{9357} & 10 & \multicolumn{1}{c}{137} & 4 & \multicolumn{1}{c}{9357} & 73 & \multicolumn{1}{c}{9357} & 71 & \multicolumn{1}{c}{5910} & \multicolumn{1}{c}{\textbf{1.58$\bm{\times}$}} & 85 \\ 
\texttt{s13207} & 604 & 1801 & \multicolumn{1}{c}{9659} & 3 & \multicolumn{1}{c}{106} & 4 & \multicolumn{1}{c}{9659} & 80 & \multicolumn{1}{c}{9659} & 5 & \multicolumn{1}{c}{9600} & \multicolumn{1}{c}{\textbf{1.01$\bm{\times}$}} & 80 \\ 
\texttt{s15850} & 649 & 2412 & \multicolumn{1}{c}{9512} & 3 & \multicolumn{1}{c}{110} & 3 & \multicolumn{1}{c}{9512} & 79 & \multicolumn{1}{c}{9512} & 8 & \multicolumn{1}{c}{6197} & \multicolumn{1}{c}{\textbf{1.53$\bm{\times}$}} & 81 \\ 
\texttt{s35932} & 1151 & 4736 & \multicolumn{1}{c}{3083} & 99 & \multicolumn{1}{c}{37} & 68 & \multicolumn{1}{c}{3083} & 100 & \multicolumn{1}{c}{3083} & 58 & \multicolumn{1}{c}{6} & \multicolumn{1}{c}{\textbf{513.83$\bm{\times}$}} & 100 \\ 
\texttt{MIPS} & 1005 & 23511 & \multicolumn{1}{c}{25000} & 0 & \multicolumn{1}{c}{796} & 0 & \multicolumn{1}{c}{25000} & 100 & \multicolumn{1}{c}{---} & --- & \multicolumn{1}{c}{1304} & \multicolumn{1}{c}{\textbf{19.17$\bm{\times}$}} & 97 \\ \cmidrule{1-14}
Avg. & 511 & 5184 & \multicolumn{1}{c}{6884} & 27.75${}^\dagger$ & \multicolumn{1}{c}{88.57} & 10${}^\dagger$ & \multicolumn{1}{c}{6884} & 83.5${}^\dagger$ & \multicolumn{1}{c}{6884} & 86.5${}^\dagger$ & \multicolumn{1}{c}{3628.85} & \multicolumn{1}{c}{\textbf{169.68$\bm{\times}$}${}^\ddagger$} & \textbf{95.75}${}^\dagger$ \\ \bottomrule
\end{tabular}
}
\\
[1mm]
${}^\dagger$The coverages are averaged over \texttt{c2670}, \texttt{c5315}, \texttt{c6288}, and \texttt{c7552}. ${}^\ddagger$The reduction is averaged over all except \texttt{MIPS}.
\end{table*}

We implemented our RL agent using \textit{PyTorch1.6} and trained it using a Linux machine with Intel 2.4 GHz CPUs and an NVIDIA Tesla K80 GPU. 
We used the SAT solver provided in the \textit{pycosat} library.
We  implemented 
the parallelized version of TARMAC in \textit{Python 3.6}.
We used \textit{Synopsys VCS} for logic simulations and for evaluating test patterns on HT-infected netlists. 
Similar to prior works (TARMAC and TGRL), for sequential circuits, we assume full scan access. 
To enable a fair comparison, we implemented and evaluated all the techniques on the same benchmarks as TARMAC and TGRL, which were provided to us by the authors of TGRL. 
They also provided us with the TGRL test patterns. 
We also performed experiments on the \texttt{MIPS} processor from OpenCores~\cite{OpenCores_MIPS} to demonstrate scalability. 
For \texttt{MIPS}, we use vectorized environment with 16 parallel processes to speed up the training.
For evaluation, we randomly inserted $100$ HTs in each benchmark and verified them to be valid using a Boolean satisfiability check.

\subsection{Trigger Coverage Performance}

In this section, we compare the trigger coverage provided by different techniques (Table~\ref{tab:my-table}). 
In addition to TARMAC and TGRL, we also compare the performance of DETERRENT with random test patterns and patterns generated from an industry-standard tool, \textit{Synopsys TestMAX}~\cite{TestMAX}. 
We used the number of patterns from TGRL as a reference for the random test patterns and TARMAC to enable a fair comparison.
For TestMAX, the number of patterns is determined by the tool in the default setting (\texttt{run\_atpg}).

Note that for \texttt{s13207}, \texttt{s15850}, and \texttt{s35932}, the netlists corresponding to the test patterns provided by the authors of TGRL were not available to us at the time of writing the manuscript. Hence, we could only evaluate the TGRL test patterns for those circuits on our benchmarks. Due to this, the trigger coverage of TGRL for these benchmarks is low. 
Additionally, TGRL does not evaluate on the \texttt{MIPS} benchmark. Hence the corresponding cells in the table are empty.
To enable a fair comparison, we have not included \texttt{s13207}, \texttt{s15850}, and \texttt{s35932} in the average test length, as well as \texttt{MIPS} in the average trigger coverages for all techniques in Table~\ref{tab:my-table}.

The results demonstrate that DETERRENT achieves better trigger coverage than all other techniques while reducing the number of test patterns. On average, DETERRENT improves the coverage over random patterns ($68\%$), TestMAX ($85.75\%$), TARMAC ($12.25\%$), and TGRL ($9.25\%$), and achieves two orders of magnitude reduction in the number of test patterns over TARMAC and TGRL ($169\times$).

\subsection{Impact of Trigger Width}
Trigger width, i.e., the number of rare nets that constitute the trigger, directly affects the stealth of the HT.
As the trigger width increases, the difficulty to activate the trigger increases exponentially.
For example, for a rareness threshold of $0.1$, if the trigger width is $4$, the probability of activating the trigger through random simulation is $10^{-4}$. 
Whereas, if the trigger width is $12$, the probability reduces to $10^{-12}$. 
Thus, it is necessary to maintain the performance with increasing trigger width.
Figure~\ref{fig:my_label1} illustrates the results for \texttt{c6288}; we chose this benchmark as TGRL provides a good trigger coverage.
With increasing trigger width, the performance of TGRL drops drastically. 
DETERRENT maintains a steady trigger coverage, demonstrating that it can activate extremely rare trigger conditions.

\begin{figure}[tb]
\centering
\includegraphics[width=0.475\textwidth,trim={0.3cm 0.3cm 0cm 0.2cm},clip]{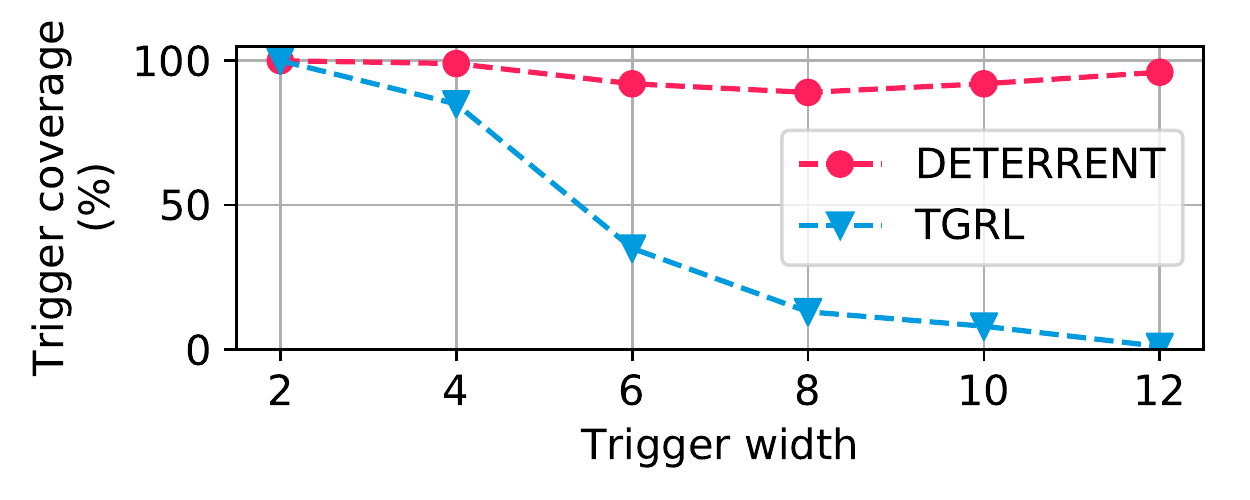}
\caption{Impact of trigger width on the trigger coverage of TGRL~\cite{pan2021automated} and DETERRENT for \texttt{c6288}.}
\label{fig:my_label1}
\end{figure}

\subsection{Trigger Coverage vs. Number of Patterns}
We now investigate the marginal impact of test patterns on trigger coverage.
To do so, we analyze the increase in trigger coverage provided by each test pattern for DETERRENT and TGRL.
Figure~\ref{fig:my_label4} demonstrates that DETERRENT obtains the maximum trigger coverage with very few patterns as opposed to TGRL.

\begin{figure}[tb]
\centering
\includegraphics[width=0.475\textwidth,trim={0 0.4cm 0 0.1cm},clip]{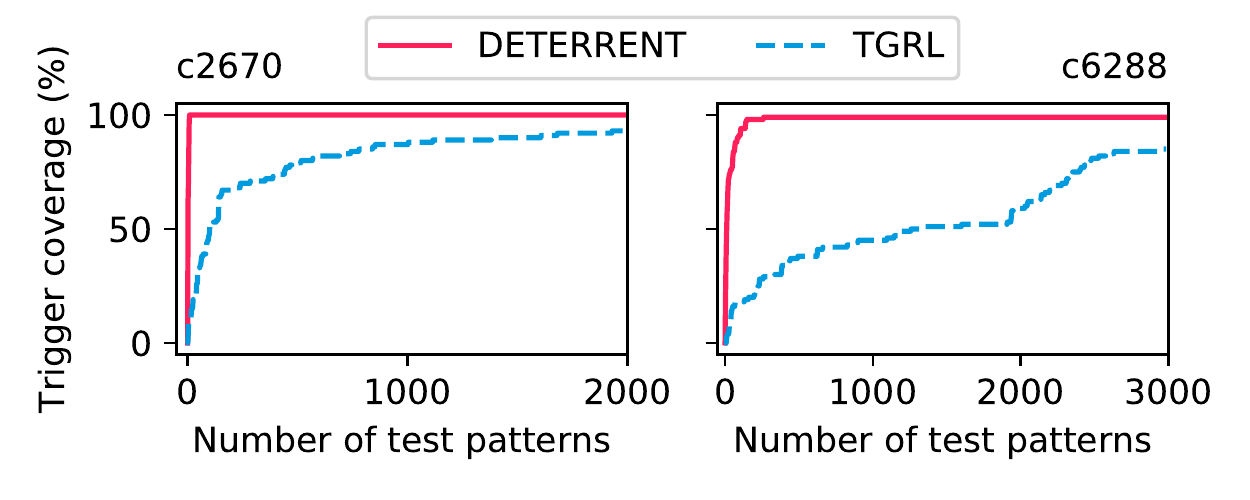}
\caption{Trigger coverage vs. test patterns comparison.}
\label{fig:my_label4}
\end{figure}

\subsection{Impact of Rareness Threshold}
Rareness threshold is the probability below which nets are classified as rare, i.e., the logic values of these nets are strongly biased towards 0 or 1. 
For a given trigger width ($\alpha$), as the rareness threshold increases, the number of rare nets increases (say by a factor of $\beta$), and so, the number of combinations possible for constructing the trigger increases by a factor of $\beta^{\alpha}$, making it much more difficult to activate.
Figure~\ref{fig:my_label5} shows that the number of rare nets increases with increasing threshold (leading to up to $64\times$ more potential trigger combinations), but DETERRENT is still able to achieve similar trigger coverage ($\leq 2\%$ drop) with less than $2500$ patterns.\footnote{The authors of TGRL did not provide us the test patterns for thresholds other than $0.1$. Hence, we do not compare with TGRL for other threshold values.} 

In another experiment, we trained the agent using rare nets for a threshold of $0.14$ and evaluated the generated test patterns on rare nets with threshold of $0.1$---the trigger coverage is $99\%$. This hints that we can train the agent for a large set of rare nets and use it to generate patterns for a subset of rare nets.

\begin{figure}[tb]
\centering
\includegraphics[width=0.475\textwidth,trim={0 0.4cm 0 0.3cm}, clip]{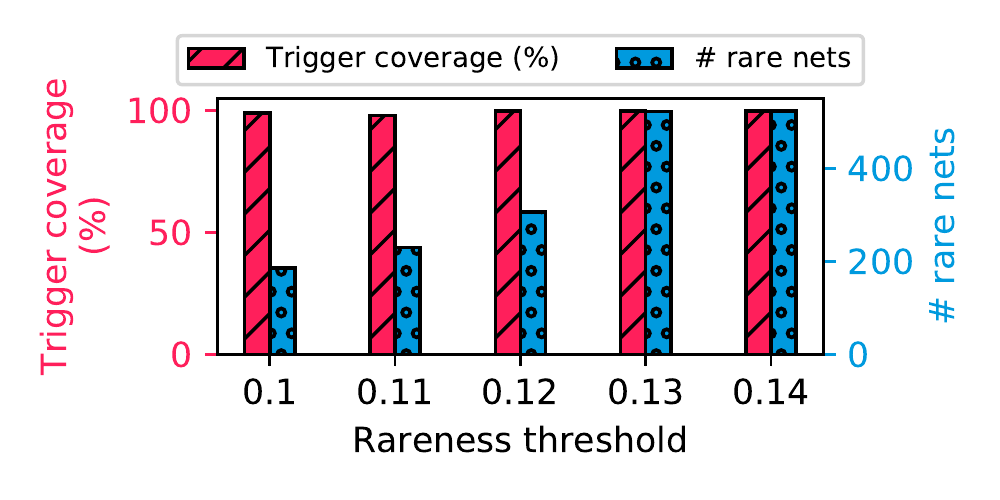}
\caption{Impact of rareness threshold on the number of rare nets and the trigger coverage of DETERRENT for \texttt{c6288}.}
\label{fig:my_label5}
\end{figure}
\section{Discussion and Future Work}
\label{sec:discussion}

\noindent\textbf{Comparison with TGRL~\cite{pan2021automated}.} Our RL agent architecture is entirely different from TGRL. 
TGRL maximizes a heuristic based on the rareness and testability of nets. 
In contrast, we identify the problem of trigger activation to be a set-cover problem and find maximal sets of compatible rare nets. 
Moreover, TGRL states and actions are test patterns generated by flipping bits probabilistically, whereas our agent's efforts are more directed by generating maximal sets of compatible rare nets.
Due to our formulation, we achieve better coverage but with orders of magnitude fewer test patterns than TGRL (see Section~\ref{sec:results}).

\noindent \textbf{Feasibility of using a SAT solver.} We use a SAT solver for the compatibility check during training and for generating test patterns from the maximal sets of compatible rare nets provided by the RL agent.
Nevertheless, our technique is scalable for larger designs (as evidenced by our results) because: 
(i)~During training, we reduce the runtime of using the SAT solver as we generate a dictionary containing the compatibility information offline in a parallelized manner.
(ii)~When generating the test patterns, we only require invoking the SAT solver $T$ times, where $T$ is the required number of test patterns.
Hence, even for large benchmarks like \texttt{MIPS}, we can generate test patterns that outperform all the HT detection techniques in less than $12$ hours.

\noindent\textbf{Meta-learning.} We generated test patterns for individual benchmarks using separate agents. 
Since the training time of our agents for all benchmarks is less than 12 hours, it is practical to use our technique. 
As part of future work, we would like to explore the principles of designing a standalone agent that can be trained on a corpus of benchmarks once and be used to generate test patterns for unseen benchmarks.
To that end, we plan to extend the current framework by using principles from meta-learning.
\section{Conclusion}
\label{sec:Conclusion}

Prior works on trigger activation for HT detection have shown reasonable trigger coverage, but they are ineffective, not scalable, or require a large number of test patterns. 
To address these limitations, we develop an RL agent to guide the search for optimal test patterns.
However, in order to design the agent, we face several challenges like inefficiency and lack of scalability. 
We overcome these challenges using different features like masking and boosting exploration of the agent. 
As a result, the final architecture generates a compact set of test patterns for designs of all sizes, including the \texttt{MIPS} processor.
Experimental results demonstrate that our agent reduces the number of test patterns by $169\times$ on average while improving trigger coverage. 
Further evaluations show that our agent is robust against increasing complexity.
Our agent maintains steady trigger coverage for different trigger widths, whereas the state-of-the-art technique's performance drops drastically. 
Our agent also maintains performance against the increasing number of possible trigger combinations. 
Although this work demonstrates the power of RL for trigger activation, the challenges related to scalability and efficiency are not specific to the current problem.
The ways in which we overcame the challenges can be used to develop better defenses for other hardware security problems.

\section*{Acknowledgments}
The work was partially supported by the National Science Foundation (NSF CNS--1822848 and NSF DGE--2039610). 
Portions of this work were conducted with the advanced computing resources provided by Texas A\&M High Performance Research Computing.

\bibliographystyle{unsrt}
\bibliography{main}

\end{document}